\documentclass{article}
\usepackage{amsmath,epsfig,bm, amssymb,subfigure,natbib,here,algorithm,algorithmic,multirow}
\newtheorem{defi}{Definition}
\newtheorem{prop}[defi]{Proposition}

\newcommand{\argmax}{\mathop{\mathrm{argmax\,}}}

\newcommand{\mathbbR}{\mathbb{R}}

\newcommand{\boldzero}{{\boldsymbol{0}}}
\newcommand{\boldone}{{\boldsymbol{1}}}

\newcommand{\boldA}{{\boldsymbol{A}}}

\newcommand{\boldF}{{\boldsymbol{F}}}
\newcommand{\boldG}{{\boldsymbol{G}}}

\newcommand{\boldI}{{\boldsymbol{I}}}

\newcommand{\boldK}{{\boldsymbol{K}}}
\newcommand{\boldL}{{\boldsymbol{L}}}

\newcommand{\boldQ}{{\boldsymbol{Q}}}

\newcommand{\boldX}{{\boldsymbol{X}}}

\newcommand{\boldc}{{\boldsymbol{c}}}

\newcommand{\boldu}{{\boldsymbol{u}}}

\newcommand{\boldx}{{\boldsymbol{x}}}
\newcommand{\boldy}{{\boldsymbol{y}}}

\newcommand{\boldalpha}{{\boldsymbol{\alpha}}}
\newcommand{\boldbeta}{{\boldsymbol{\beta}}}

\newcommand{\boldmu}{{\boldsymbol{\mu}}}

\newcommand{\boldpsi}{{\boldsymbol{\psi}}}

\newcommand{\boldGamma}{{\boldsymbol{\Gamma}}}

\newcommand{\boldSigma}{{\boldsymbol{\Sigma}}}

\newcommand{\calA}{{\mathcal{A}}}

\newcommand{\calI}{{\mathcal{I}}}

\newcommand{\calX}{{\mathcal{X}}}
\newcommand{\calY}{{\mathcal{Y}}}

\usepackage{graphicx,url}

\def\CUT#1{{}}

\advance\oddsidemargin-0.85in
\date{\today}
\title{N$^3$LARS: Minimum Redundancy Maximum Relevance Feature Selection for Large and  High-dimensional Data} 

\author{Makoto Yamada, Avishek Saha, Hua Ouyang, Dawei Yin, and Yi Chang \\
Yahoo Labs, 701 1st Ave., Sunnyvale, 94089, USA\\
\texttt{\{makotoy,avishek2,houyang,daweiy,yichang\}@yahoo-inc.com}}

\begin{document}
\maketitle

\begin{abstract}
We propose a feature selection method that finds non-redundant features from a large and high-dimensional data  in nonlinear way. Specifically, we propose a nonlinear extension of the  \emph{non-negative least-angle regression} (LARS) called $\text{N}^3$LARS, where the similarity between input and output is measured through the normalized version of the Hilbert-Schmidt Independence Criterion (HSIC).  An advantage of $\text{N}^3$LARS is that it can easily incorporate with map-reduce frameworks such as Hadoop and Spark. Thus, with the help of distributed computing, a set of features can be efficiently selected from a large and high-dimensional data.  Moreover, $\text{N}^3$LARS is a convex method and can find a global optimum solution. The effectiveness of the proposed method is first demonstrated through feature selection experiments for classification and regression with small and high-dimensional datasets. Finally, we evaluate our proposed method over a large and high-dimensional biology dataset.
\end{abstract}

\section{Introduction}

Feature selection is an important machine learning problem, and it is widely used for various types of applications such as gene selection from microarray data \citep{huang2010variable}, document categorization \citep{CIKM:Forman:2008}, and prosthesis control \citep{IEEEBIO:Shenoy+etal:2008}, to name a few. The feature selection problem is a traditional and popular machine learning problem, and thus there exist many methods including the \emph{least absolute shrinkage and selection operator} (Lasso) \citep{JRSSB:Tibshirani:1996} and the spectral feature selection (SPEC) \citep{zhao2011spectral}. 

Feature selection methods can be mainly categorized into two classes: \emph{Wrapper}  and \emph{Filter}  methods \citep{JMLR:Guyon+etal:2003,brown2012conditional}. For wrapper methods, it finds a set of features so that the classification/regression performance becomes higher. In contrast, filter methods select features which are independent of the choice of successive classifier or regressor. Namely, it would be more appropriate for interpreting the selected features than wrapper type methods. 

 Recently, a \emph{wrapper} based large-scale feature selection method called the \emph{feature generation machine} (FGM) was proposed \citep{JMLR:v15:tan14a}. However, to the best of our knowledge,  there is a few \emph{filter} based methods for large and high-dimensional setting, in particular for nonlinear and dense setting. Moreover, existing methods are \emph{maximum relevance} MR-based methods which selects $m$ features with the largest relevance to the output \citep{PAMI:Peng+etal:2005}. MR-based methods are simple yet efficient and can be easily applicable to high-dimensional and large sample problems. However, since MR-based approaches only use input-output relevance and not use input-input relevance, they tend to select redundant features. Thus, in this paper, we focus on a filter type feature selection approaches and propose a method to find a non-redundant feature set from  a large $n$ and  high-dimensional $d$ in nonlinear manner.

 To deal with a large and high-dimensional problem, we propose a \emph{Nonlinear} extension of the \emph{Non-Negative least angle regression} (NN-LARS) \citep{efron2004least} called $\text{N}^3$LARS. More specifically, we reformulate the Hilbert-Schmidt Independence Criterion Lasso (HSIC Lasso) optimization problem~\citep{yamada2014high,hetransductive} with NN-LARS~\citep{efron2004least},  where the similarity between input and output is measured through a normalized variant of HSIC~\citep{DBLP:journals/jmlr/CortesMR12}. 
There are a couple of advantages in $\text{N}^3$LARS over HSIC Lasso in a large and high-dimensional  setting. First, $\text{N}^3$LARS finds a feature one by one and get $m$ features with only $m$ steps, while HSIC Lasso needs to select a regularization parameter for obtaining $m$ features and this may require several runs of HSIC Lasso. Second, since it is possible to draw regularization path in $\text{N}^3$LARS, it may be helpful for feature analysis (such as in biology). Finally, the proposed method can be easy implemented in distributed computation frameworks with Hadoop. This is very helpful property for practitioners.  Through benchmark feature selection datasets, we first show that the proposed approach compares favorably with existing filter type feature selection methods. Then, we also evaluate $\text{N}^3$LARS over a large and high-dimensional biology dataset. 

\noindent {\bf Contribution:}
Our main contribution of this paper is to propose a \emph{filter} based nonlinear method to select features from large and high-dimensional datasets. In particular, we focus on selecting non-redundant features from dense data in non-linear manner. The HSIC Lasso only considered small-$n$ and high-$d$ scenarios, and it does not scale well to large-$n$ and high-$d$ data. In contrast, N3LARS "scales easily" for large-$n$ problems based on a combination of Nystr\"{o}m’s method and map-reduce. To the best of our knowledge, \emph{filter} based nonlinear feature selection from large-$n$ and high-$d$ data is not well studied in literature and we propose the first scalable algorithm in this realm. Finally, we establish the relationship between the proposed method and the theoretically well justified feature screening method \citep{balasubramanian2013ultrahigh}.  


\section{Background}
In this section, we first formulate the supervised feature selection problem. Then, we review Lasso based feature selection methods and point out their limitations.

Let $\boldX = [\boldx_1, \ldots, \boldx_n] = [\boldu_1, \ldots, \boldu_d]^\top \in \mathbbR^{d \times n}$ denotes the input data, 
$\boldy = [y_1, \ldots, y_n]^\top \in \mathbbR^{n}$ denotes the output data or labels
and suppose that we are given $n$ independent and identically distributed (i.i.d.) paired samples $\{(\boldx_i, y_i)~|~\boldx_i \in \calX,~~y_i \in \calY,~i=1,\ldots,n\}$ drawn from a joint distribution with density $p(\boldx, y)$. Note, $\calY$ could be either continuous (i.e., regression) or categorical (i.e., classification). 
For the purposes of this paper consider $\boldx$ to be a dense vector.  

The goal of supervised feature selection is to find $m$ features ($m < d$) of input vector $\boldx$ that are responsible for predicting output $y$.

{\bf Lasso} (Least Absolute Shrinkage and Selection Operator)~\citep{JRSSB:Tibshirani:1996} is a computationally efficient linear feature selection method that optimizes the cost function:
\[
\min_{\boldalpha \in \mathbbR^d} \hspace{0.3cm} \frac{1}{2}\|\boldy - \boldX^\top \boldalpha \|^2_2 + \lambda \|\boldalpha\|_1,
\]
where $\boldalpha = [\alpha_1, \ldots, \alpha_d]^\top$ is a coefficient vector, $\alpha_k$ is the regression coefficient of the $k$-th feature, $\|\cdot\|_1$ and $\|\cdot\|_2$ are the $\ell_1$- and $\ell_2$-norms, and $\lambda > 0$ is the regularization parameter. 
The $\ell_1$-regularizer in Lasso tends to give a sparse solution and the regression coefficients for irrelevant features become zero, and Lasso is useful especially when the number of features is larger than the number of training samples \citep{JRSSB:Tibshirani:1996}.  However, Lasso can only capture linear dependency, and this is a critical limitation. 

{\bf Instance-wise Non-linear Lasso} was introduced~\citep{IEEENN:Roth:2005} to deal with non-linearity in Lasso where the original instance $\boldx$ is transformed by a non-linear function $\boldpsi(\cdot): \mathbbR^d \rightarrow \mathbbR^{d'}$.
The corresponding optimization problem is given as
\[
\min_{\boldbeta \in \mathbbR^n} \hspace{0.3cm} \frac{1}{2}\| \boldy  - \boldA \boldbeta \|^2_2 + \lambda \|\boldbeta\|_1,
\]
where $A_{i,j} = \boldpsi(\boldx_i)^\top \boldpsi(\boldx_j) = A(\boldx_i, \boldx_j)$, $\boldbeta = [\beta_1, \ldots, \beta_n]^\top$ is a regression coefficient vector, and $\beta_j$ is a coefficient of the $j$-th basis $A(\boldx,\boldx_j)$.
The instance-wise non-linear Lasso can capture non-linearity and gives a sparse solution in terms of instances, but it cannot be used for feature selection.


{\bf HSIC Lasso}~\citep{yamada2014high} can select features from high-dimensional data in a nonlinear manner by optimizing the cost function: 
\begin{align}
\label{eq:hsic-form}
\min_{\boldalpha \in \mathbbR^d}  & \hspace{0.3cm}
\|\bar{\boldL} - \sum_{k = 1}^{d} \alpha_k \bar{\boldK}^{(k)} \|^2_{\textnormal{Frob}}  +  \lambda \|\boldalpha\|_1 \nonumber, \\
&\textnormal{s.t.}  \hspace{0.3cm} \alpha_1,\ldots,\alpha_d \geq 0,
\end{align}
where $\bar{\boldK}^{(k)} = \boldGamma \boldK^{(k)} \boldGamma$ and $\bar{\boldL} = \boldGamma \boldL \boldGamma$ are centered and normalized Gram matrices, $\boldK_{i,j}^{(k)} = K(x_{k,i}, x_{k,j})$ and $\boldL_{i,j} = L(y_{i}, y_{j})$ are Gram matrices, $K(x,x')$ and $L(y,y')$ are kernel functions, $\boldGamma = \boldI_{n} - \frac{1}{n}\boldone_n \boldone_n^\top$ is the centering matrix, $\boldI_n$ is the $n$-dimensional identity matrix, $\boldone_n$ is the $n$-dimensional vector with all ones, and $\|\cdot\|_{\textnormal{Frob}}$ is the Frobenius norm.
The objective function of HSIC Lasso can also be rewritten as~\citep{yamada2014high}:
\begin{align}
\label{eq:HSIC Lasso-HSIC-relation-ext}
C - 2\sum_{k = 1}^d\alpha_k {\textnormal{HSIC}}(\boldu_k,\boldy) \nonumber + \sum_{k,l = 1}^d \alpha_k \alpha_l {\textnormal{HSIC}}(\boldu_k,\boldu_l),
\end{align}
where $C = {\textnormal{HSIC}}(\boldy,\boldy) = \text{tr}(\bar{\boldL}\bar{\boldL})$ and ${\textnormal{HSIC}}(\boldu_k,\boldu_{k'}) = \text{tr}(\bar{\boldK}^{(k)}\bar{\boldK}^{(k')})$ are kernel-based independence measure called as HSIC \citep{ALT:Gretton+etal:2005}, and $\boldX = [\boldu_1, \ldots, \boldu_d]^\top$. Note that HSIC always takes a non-negative value, and is zero if and only if two random variables are statistically independent when a \emph{universal reproducing kernel} \citep{JMLR:Steinwart:2001} such as the Gaussian kernel is used.
 
If the $k$-th feature $\boldu_k$ has high dependence on output $\boldy$,
${\textnormal{HSIC}}(\boldu_k,\boldy)$ becomes a large value
and thus $\alpha_k$ should also be large. On the other hand, if $\boldu_k$ and $\boldy$ are independent,
${\textnormal{HSIC}}(\boldu_k,\boldy)$ is close to zero; $\boldu_k$ tends to be not selected by the $\ell_1$-regularizer. Furthermore, if $\boldu_k$ and $\boldu_l$ are strongly dependent (i.e., redundant features), ${\textnormal{HSIC}}(\boldu_k,\boldu_l)$ is large 
and thus either of $\alpha_k$ and $\alpha_l$ tends to be zero. That is, non-redundant features that have strong dependence on output $\boldy$ tend to be selected by the HSIC Lasso. 

HSIC Lasso outperforms existing feature selection methods in small and high-dimensional setting \citep{yamada2014high}.  However, since HSIC Lasso needs $O(dn^2)$ memory space, it is not suited for large $d$ and large $n$ setting.  In \citep{yamada2014high}, a table lookup based approach was proposed to reduce memory usage but the computational cost was still large. Moreover, HSIC Lasso needs to tune  the regularization parameter $\lambda$ which is usually difficult to set manually to obtain $m$ features. Finally, statistical properties of HSIC Lasso are not well studied.


\section{Proposed Method}

In this paper, to efficiently solve a large and high-dimensional feature selection problem, we first propose \emph{Nonlinear} extension of the \emph{Non-Negative} \emph{least angle regression} (NN-LARS)  \citep{efron2004least} called $\text{N}^3$LARS. Then, we propose a distributed computing approach to scale up the proposed algorithm for a large and high-dimensional problem. Moreover, to justify the statistical property of $\text{N}^3$LARS, we establish the relationship between the proposed method and the feature screening method in \citep{balasubramanian2013ultrahigh}.

\subsection{Nonlinear Non-Negative LARS ($\text{N}^3$LARS)} 
It is well known that non-negative Lasso (a.k.a, positive Lasso) can be efficiently solved by non-negative LARS (NN-LARS)\citep{efron2004least,morup2008approximate}. That is, the HSIC Lasso problem can be efficiently solved via NN-LARS. Note that NN-LARS requires the inputs $\boldK^{(1)}, \ldots,\boldK^{(d)}$ be standardized to have mean 0 and unit length and the output $\boldL$ have mean 0. Thus, in our formulation, we normalized each input and output as $\widetilde{\boldK} = \bar{\boldK}/\|\bar{\boldK}\|_{\text{Fro}}$, $\widetilde{\boldL} = \bar{\boldL}/\|\bar{\boldL}\|_{\text{Fro}}$ such that $\boldone_n^\top \widetilde{\boldK}\boldone_n = 0$, $\boldone_n^\top \widetilde{\boldL}\boldone_n = 0$, and $\|\widetilde{\boldK}\|_{\text{Fro}}^2 = 1$.  With these modifications the optimization problem of the proposed method  can be written as
\begin{align}
\min_{\boldalpha \in \mathbbR^d}~
\|\widetilde{\boldL}\! -\! \sum_{k = 1}^{d} \alpha_k \widetilde{\boldK}^{(k)} \|^2_{\textnormal{Frob}} \! + \! \lambda \|\boldalpha\|_1, ~\textnormal{s.t.}~\alpha_1,\ldots,\alpha_d \geq 0.
\end{align}
We call this method Nonlinear Non-Negative LARS ($\text{N}^3$LARS).  
It is worth pointing out the objective function of $\text{N}^3$LARS can be reframed as
\begin{align}
C - 2\sum_{k = 1}^d\alpha_k {\textnormal{NHSIC}}(\boldu_k,\boldy) \nonumber + \sum_{k,l = 1}^d \alpha_k \alpha_l {\textnormal{NHSIC}}(\boldu_k,\boldu_l),
\end{align}
where $\text{NHSIC}(\boldu, \boldy) = \text{tr}(\widetilde{\boldK} \widetilde{\boldL})$ is the normalized version of HSIC (a.k.a, the centered kernel target alignment \citep{DBLP:journals/jmlr/CortesMR12}) and $C = {\textnormal{NHSIC}}(\boldy,\boldy)$. Since the NHSIC score is zero if and only if two random variables are independent (See the proof in Proposition 1) and have similar form of HSIC Lasso, we can obtain non-redundant features.

Let $\calA$ be the indices of the active set and $\calI$ the indices of the inactive set. The $\text{N}^3$LARS algorithm is summarized in Algorithm 1, where $\boldalpha_{\calA} \in \mathbbR^{|\calA|}$, $\boldone \in \mathbbR^{|\calA|}$ are vectors with all ones, $[ \boldQ_{\calA} ]_{i,j} = \text{NHSIC}(\boldu_{\calA,i}, \boldu_{\calA,j})$, and $\boldu_{\calA,i}$ is a feature vector selected at $i$-th step. 

\begin{algorithm}[t]
	\caption{$\text{N}^3$LARS}
	\begin{algorithmic}
		\STATE {\bf Initialize}: $\boldalpha = \boldzero$ and $\calA = []$.
		\WHILE{$|\calA| < m$}
		\STATE /* Select $m$ features */
		\FOR{$k = 1 \ldots d$}
		  \STATE /* compute negative gradient */
		  \STATE $\widehat{c}_k = \text{NHSIC}(\boldu_k, \boldy) - \sum_{i = 1}^d \alpha_i \text{NHSIC}(\boldu_k, \boldu_i)$
		\ENDFOR
		\STATE {\bf Find feature index}: $j = \argmax_{\boldc_{\calI}} \widehat{c}_k > 0$.
		\STATE {\bf Update sets}: $\calA = [\calA~j], \calI = \calI \backslash j$
		\STATE {\bf Update coefficients}:
			\begin{align*}
			\boldalpha_{\calA} &= \boldalpha_{\calA} + \widehat{\mu} \boldQ_{\calA}^{-1} \boldone, \\
			\widehat{\mu} &= \min_{\mu} \left\{ \begin{array}{ll}
			\exists \ell \in I: \widetilde{\boldc}_\ell = \widetilde{\boldc}_{\calA}& \\
			\widetilde{\boldc}_{\calA} = \boldzero& \\
			\end{array} \right.,
			\end{align*}
               \ENDWHILE
	\end{algorithmic}
\end{algorithm}

An advantage of the LARS based formulation over Lasso is that LARS can find $m$ features by iterating over $m$ steps while Lasso requires fine tuning the regularization parameter $\lambda$ to obtain $m$ features. For high-dimensional and small size problems, it is reasonable to run HSIC Lasso several times to obtain $m$ features by changing $\lambda$. However, a large and high-dimensional problems, tuning the regularization parameter is very expensive. In addition, since the proposed method can be regarded as an NN-LARS algorithm \citep{efron2004least,morup2008approximate},  we can obtain an entire regularization path for the cost of an ordinary least squares method. 

If we solve Eq.(2) without the $\ell_1$ regularizer, the time complexity is $O(d^2 n^3)$ ($O(n^3)$ for NHSIC), and this is infeasible when both $d$ and $n$ are large (this is the complexity of QPFS with HSIC \citep{DBLP:journals/jmlr/CortesMR12}). However, if we add the $\ell_1$ regularizer and solve with $\text{N}^3$Lars, we do not need to compute all NHSIC values and the complexity becomes $O(mdn^3)$ ($m \ll d$ is the number of selected features). This is much better than $O(d^2n^3)$. However, the computational cost of $\text{N}^3$LARS grows rapidlly when both $d$ and $n$ increase (see Figure \ref{fig:illustrative_example}(b)).  To deal with this computational issues, we first introduce the Nystr\"{o}m approximation to reduce the computational cost of NHSIC \citep{book:Schoelkopf+Smola:2002}. We then use map-reduce to take advantage of parallelism and further speed up $\text{N}^3$LARS. 

\subsection{Nystr\"{o}m Approximation for NHSIC}
We start  introducing kernels used in this paper.

\vspace{0.1in}
\noindent {\bf Kernel types:}
A universal kernel such as the Gaussian kernel  allows HSIC to detect dependence between two random variables \citep{ALT:Gretton+etal:2005}. Moreover, it has been shown that the delta kernel is useful for multi-class classification problems \citep{song2012feature}. Thus, we use the Gaussian kernel for inputs $\boldx$. For output kernels, we use the Gaussian kernel for regression cases and the delta kernel for classification problems.

For input $x$, we first normalize the input $x$ to have unit standard deviation, and then use the Gaussian kernel, $K(x,x') = \exp \left(-\frac{( x - x' )^2}{2\sigma_{\mathrm x}^2} \right)$, where $\sigma_\mathrm{x}$ is the Gaussian kernel width. 

In regression cases (i.e., $y\in\mathbbR$), we normalize an output $y$ to have unit standard deviation, and then use the Gaussian kernel,
$L(y,y') = \exp \left(-\frac{( y - y' )^2}{2\sigma_{\mathrm y}^2} \right)$, where  $\sigma_\mathrm{y}$ is the Gaussian kernel width.  In this paper, we use $\sigma_{\mathrm{x}}^2 =1$ and $\sigma_{\mathrm{y}}^2 = 1$.

In classification cases (i.e., $y$ is categorical), we use the delta kernel for $y$, 
\begin{eqnarray*}
L(y,y') = \left\{ \begin{array}{ll}
{1}/{n_{y}} & \textnormal{if}~y = y', \\
0 & \textnormal{otherwise}, \\
\end{array} \right.
\end{eqnarray*} 
 where $n_{y}$ is the number of samples in class $y$. 

\vspace{0.1in}
\noindent {\bf Nystr\"{o}m approximation:} To reduce the computational cost of Gram matrices, we use the Nystr\"{o}m approximation for NHSIC as
\begin{align}
\text{NHSIC}(\boldu,\boldy) = \text{tr}((\boldF^\top \boldG)^2),
\end{align}
where $\boldF = \boldGamma \boldK_{nb}\boldK_{bb}^{-1/2}/(\text{tr}((\boldK_{nb}\boldK_{bb}^{-1/2})^2))^{1/4}$, $\boldK_{nb} \in \mathbbR^{n \times b}$, $\boldK_{bb} \in \mathbbR^{b \times b}$,  $\boldG = \boldGamma \boldL_{nb}\boldL_{bb}^{-1/2}/(\text{tr}((\boldL_{nb}\boldL_{bb}^{-1/2})^2))^{1/4}$, $\boldL_{nb} \in \mathbbR^{n \times b}$, $\boldL_{bb} \in \mathbbR^{b \times b}$, and $b$ is the number of basis function.  In this paper, we use $\boldu_b = [-5, -4.47, \ldots, 4.47, 5.0]^\top \in \mathbbR^{20}$, where $b = 20 \ll n$.


For the output matrix $\boldG$ in regression, we can similarly use the above technique to approximate the Gram matrix. For classification, we can simply compute $\boldG$ as
\begin{align*}
\boldG_{k,j} &= \left\{ \begin{array}{ll}
\frac{1}{\sqrt{n_k}}& (k = y_{j}) \\
0& (\text{otherwise})\\
\end{array} \right., 
\end{align*}
where $\boldG \in \mathbbR^{C \times n}$. The computational complexity of NHSIC for regression problem is $O(bn^2)$, and the entire computational complexity of $\text{N}^3$LARS becomes $O(mdbn^2)$. In particular for large $n$, the Nyst\"{o}m approximation is very helpful.

\subsection{Distributed computation with Map-Reduce}

The Nystr\"{o}m approximation is useful for large $n$. However, for high-dimensional cases (i.e., $d$ is also very large), the computation cost of $\text{N}^3$LARS is still large. To deal with this,  we separately compute $c_k$ in $\text{N}^3$LARS in parallel with distributed computing such as Hadoop. Here, we show pseudo code for Hadoop streaming.

\begin{description}
\item[{Preparation:} ]Compute $\boldG$ and store it to ``output'' file.
\item[{\bf Step1:}] For each feature vector $\boldu_k$, we compute $\boldF_k$, $\text{NHSIC}(\boldu_k, \boldy)$, and output a key-value pair $<k, \boldF_k \in \mathbbR^{n \times b}>$. 
\item[{\bf Step2:}] With given $<j = \argmax(\boldc_I), \boldF_j>$, we compute $\text{NHSIC}(\boldu_k, \boldu_j)$ and output a key-value pair $<k, \text{NHSIC}(\boldu_k, \boldu_j)>$. 
\item[{\bf Step3:}] With given $<k, \text{NHSIC}(\boldu_k, \boldu_j)>$ and $<-1, \text{NHSIC}(\boldu_k, \boldy)>$, we update $\boldalpha$.
\item[{\bf Step4:}]With given $<k, \text{NHSIC}(\boldu_k, \boldu_j)>$, $<-1, \text{NHSIC}(\boldu_k, \boldy)>$, and $\boldalpha^{\text{new}}$, we compute $\boldc$.
\end{description}

We repeat Step 2 to 4 until obtaining $m$ features. In this paper, we use Python and Hadoop streaming to implement the proposed method. With using map-reduce, the complexity is further reduced to $O(mdbn^2/P)$, where $P$ is the number of mappers.  To further speed up, we could use a faster distributed computation framework such as SPARK and JAVA. 

\subsection{Relation to High-dimensional feature screening method}
In this section we establish a relation between the proposed method and the feature screening method \citep{balasubramanian2013ultrahigh}. 

The high-dimensional feature screening method is a general framework for model-free feature screening. The idea of this method is to rank the covariates between the input random variables $\boldu$ and the output response $\boldy$ according to some degree of dependence. For example, one can choose $\text{NHSIC}$ as such a measure and rank the $d$ features $\boldu_1,\ldots,\boldu_d$ according to their values $\text{NHSIC}(\boldu_k,\boldy)$. The top $m$ features are then regarded as relevant.

The relation between this method and our proposed method can be expressed as follows.
\begin{prop}
If any pair of features $\boldu_k$ and $\boldu_{k'}$ are assumed to be independent, then there exist a pair $(\lambda, m)$ such that the top $m$ features obtained by \citep{balasubramanian2013ultrahigh} is the same of that obtained by solving Eq.~(2).
\end{prop}
\begin{proof}
Since the two features $\boldu_k$ and $\boldu_{k'}$ are assumed to be independent, according to the result of independent tests for HSIC (e.g. Theorem 4 in \citep{ALT:Gretton+etal:2005}), $\text{HSIC}(\boldu_k, \boldu_{k'})=0$. Using the definition of NHSIC: $\text{NHSIC}(\boldu_k, \boldu_{k'}) = \frac{\text{HSIC}(\boldu_k, \boldu_{k'})}{\sqrt{\text{tr}(\bar{\boldK}^{(k)} \bar{\boldK}^{(k)} )} \sqrt{\text{tr}(\bar{\boldK}^{(k')}  \bar{\boldK}^{(k')})}} = 0$ if $k \neq k'$ and $\text{NHSIC}(\boldu_k, \boldu_k') = 1$ if $k = k'$.  Hence solving the objective function Eq.~(2) is equivalent to:
\begin{align}\label{eq:prop}
\begin{split}
\max_{\boldalpha \in \mathbbR^{d}} &\hspace{0.3cm} \sum_{k=1}^d \alpha_k \text{NHSIC}(\boldu_k, \boldy) - \frac{1}{2}\|\boldalpha\|_2^2 - \frac{\lambda}{2} \|\boldalpha\|_1,\\
\text{s.t.} & \hspace{0.3cm} \alpha_1,\ldots,\alpha_d \geq 0.
\end{split}
\end{align} 
Next we prove that the $\alpha$ associated with the largest NHSIC values are the same as the largest $\alpha_k$ in the solution of \ref{eq:prop}. We prove it by contradiction. Suppose there exists a pair $(i,j)$ such that $\text{NHSIC}(\boldu_i,\boldy) > \text{NHSIC}(\boldu_j,\boldy)$ and $\alpha_i < \alpha_j$. Then one can simply switch the values of $\alpha_i$ and $\alpha_j$ and obtains a higher value in the objective function \ref{eq:prop}. This contradiction states that the largest $\alpha_k$ correspond to the largest values $\text{NHSIC}(\boldu_k,\boldy)$.
\end{proof}
The above proposition shows that the connection to feature screening method \citep{balasubramanian2013ultrahigh}.

Note, in \citep{balasubramanian2013ultrahigh}, since the feature screening method tends to select redundant features, an iterative screening approach is used to filter out redundant features. Specifically, they first screen $m$ features from entire $d$ features and then select $m'$ features out of $m$ features by an iterative screening method. This technique can also be used for $\text{N}^3$LARS.

\section{Related Work}
\label{sec:existing}

In this section, we review existing feature selection methods and discuss their relation to the proposed approach.

{\bf Maximum Relevance (MR)} feature selection is a popular approach that selects $m$ features with the largest relevance to the output \citep{PAMI:Peng+etal:2005}. The feature screening method \citep{balasubramanian2013ultrahigh} is also an MR-method. Usually, the mutual information and a kernel-based independence measures such as HSIC are used as the relevance score \citep{ALT:Gretton+etal:2005}. MR-based methods are simple yet efficient and can be easily applicable to high-dimensional and large sample problems. However, since MR-based approaches only use input-output relevance and not use input-input relevance, they tend to select redundant features (i.e., the selected features can be very similar to each other). As a result it may not help in improving overall classification/regression accuracy and interpretability. 

{\bf Minimum Redundancy Maximum Relevance (mRMR)}~\citep{PAMI:Peng+etal:2005} was proposed to deal with the feature redundancy problem; it selects features that have high relevance with respect to an output and are non-redundant. It has been experimentally shown that mRMR outperforms MR feature selection methods~\citep{PAMI:Peng+etal:2005}. Moreover, there exists an off-the-shelf C++ implementation of mRMR, and it can be applicable to a large and high dimensional feature selection. {\bf Fast Correlation based filter (FCBF)} can also be regarded as an mRMR method, in which it uses symmetrical uncertainty to calculate
dependences of features and finds best subset using backward selection with sequential search strategy \citep{yu2003feature}. Note that, it has also been reported that FCBF compares favorably with mRMR \citep{senliol2008fast}. However, both mRMR and FCBF use greedy search strategies such as forward selection/backward elimination and tends to produce locally optimal feature set. Another potential weakness of mRMR is that mutual information is approximated by Parzen window estimation which is not accurate in particular when the number of training samples is small~\citep{BMCBio:Suzuki+etal:2009a}.

To obtain a globally optimal feature set, a convex relaxed version of mRMR called the {\bf Quadratic Programming Feature Selection (QPFS)} was proposed in~\citep{JMLR:Rodriguez+etal:2010}. An advantage of QPFS over mRMR is that it can find a globally optimal solution by just solving a QP problem. The authors showed that QPFS compares favorably with mRMR for large sample size but low-dimensional cases (e.g., $d < 10^3$ and $n > 10^4$). However, QPFS tends to be computationally expensive for large and high-dimensional  cases, since QPFS needs to compute $d(d -1)/2$ mutual information scores. To deal with the computational problem in QPFS, a Nystr\"{o}m approximation based approach was proposed~\citep{JMLR:Rodriguez+etal:2010}, and it has been experimentally shown that QPFS with Nystr\"{o}m approximation compares favorably with mRMR both in accuracy and time.  However, for large and high-dimensional settings, computational cost for mutual information is still very high. Also, similar to mRMR, the mutual information approximation may not be so accurate. 

Forward/Backward elimination based feature selection with {\bf HSIC} (FOHSIC/BAHSIC) is also a popular feature selection method \citep{song2012feature}.   An advantage of HSIC-based feature selection over mRMR is that the HSIC score can be accurately estimated. Moreover, HSIC can be implemented very easily. However, similar to mRMR, it selects features using greedy search algorithm and tends to have a locally optimal feature set. To obtain a better feature set, {\bf HSFS} was proposed~\citep{DBLP:conf/icml/MasaeliFD10} as a continuously relaxed version of FOHSIC/BAHSIC that could be solved by limited-memory BFGS (L-BFGS)~\citep{book:Nocedal:2003}.  However, HSFS is a non-convex method and restarting from many different initial points would be necessary to select good features which is computationally expensive. 

For small and high-dimensional feature selection problems (e.g., $n < 100$ and $d> 10^4$), $\ell_1$ regularized based approaches such as Lasso are useful \citep{JRSSB:Tibshirani:1996,zhao2010efficient}. In addition, Lasso is known to scale well with both number of samples as well as dimensionality \citep{JRSSB:Tibshirani:1996,zhao2010efficient}. However, Lasso can only capture linear dependency between input features and output values. To handle non-linearity {\bf HSIC Lasso} was proposed recently \citep{yamada2014high}. In HSIC Lasso, with specific choice of kernel functions, non-redundant features with strong statistical dependence on output values can be found in terms of HSIC by simply solving a Lasso problem. Although, empirical evidence~\citep{yamada2014high} shows that HSIC Lasso outperforms most existing feature selection methods, in general HSIC Lasso tends to be expensive compared to simple Lasso when the number of samples increases. Moreover, statistical properties of HSIC Lasso is not well studied. 

{\bf Sparse Additive Models (SpAM)} are useful for high-dimensional feature selection problems~\citep{ravikumar2009sparse,NIPS2008_0329, raskutti2012minimax,suzuki2012fast} and	can be efficiently solved by the \emph{back-fitting} algorithms~\citep{ravikumar2009sparse} resulting in globally optimal solutions. Also, statistical properties of the SpAM estimator are well studied~\citep{ravikumar2009sparse}, and SpAM is closely related to multiple kernel learning (MKL) methods \citep{NIPS2008_0171}.  However, a potential weakness of SpAM is that it can only deal with additive models and may not work well for non-additive models. Another weakness of SpAM is that it needs to optimize $nd$ variables and hence tends to be computationally expensive.

\begin{figure*}[t!]
\begin{center}
\begin{minipage}[t]{0.45\linewidth}
\centering
  {\includegraphics[width=0.99\textwidth]{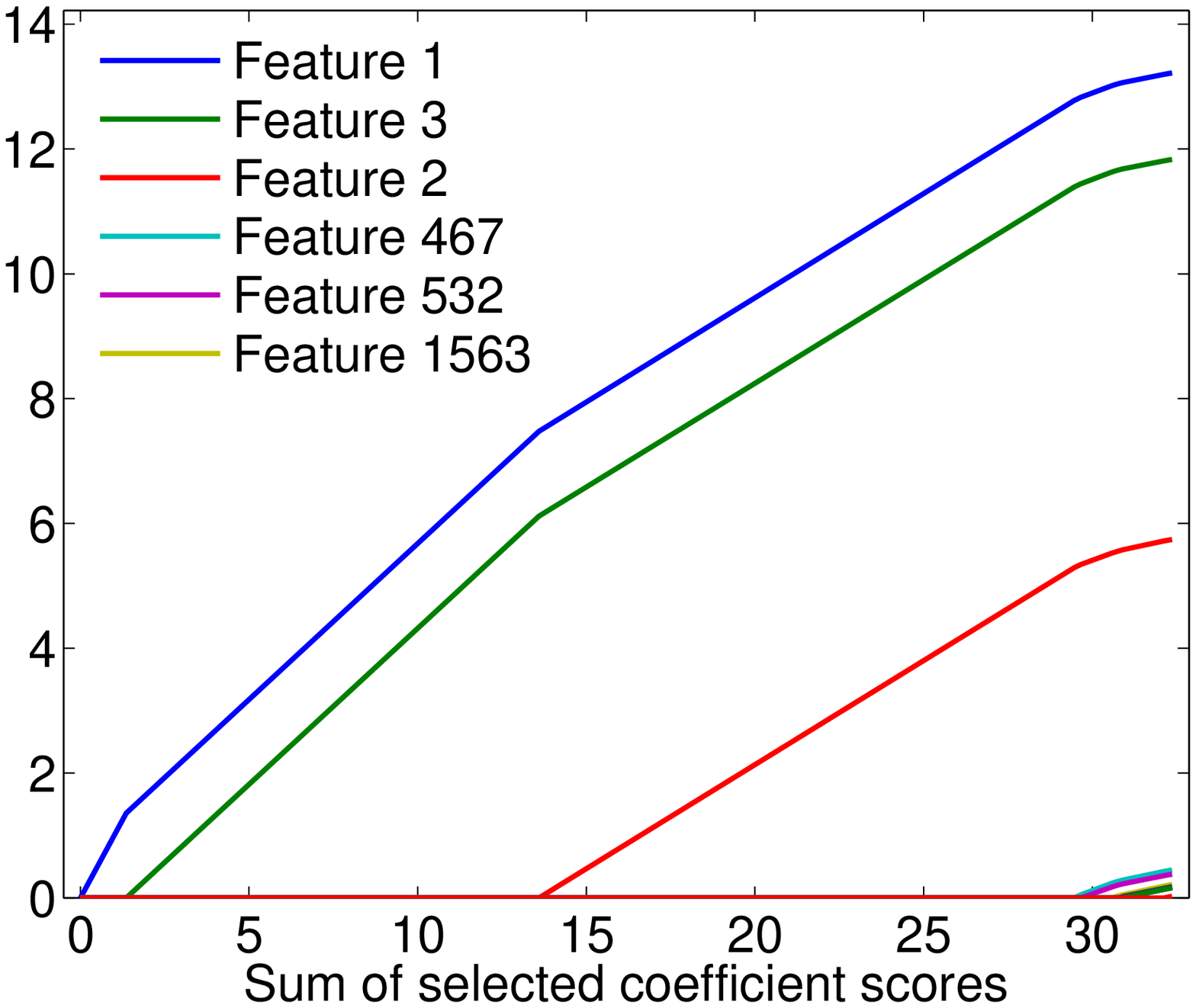}} 
  (a) Regularization path of $\text{N}^3$LARS.\\ \vspace{-0.10cm}
\end{minipage}
\begin{minipage}[t]{0.49\linewidth}
\centering
  {\includegraphics[width=0.99\textwidth]{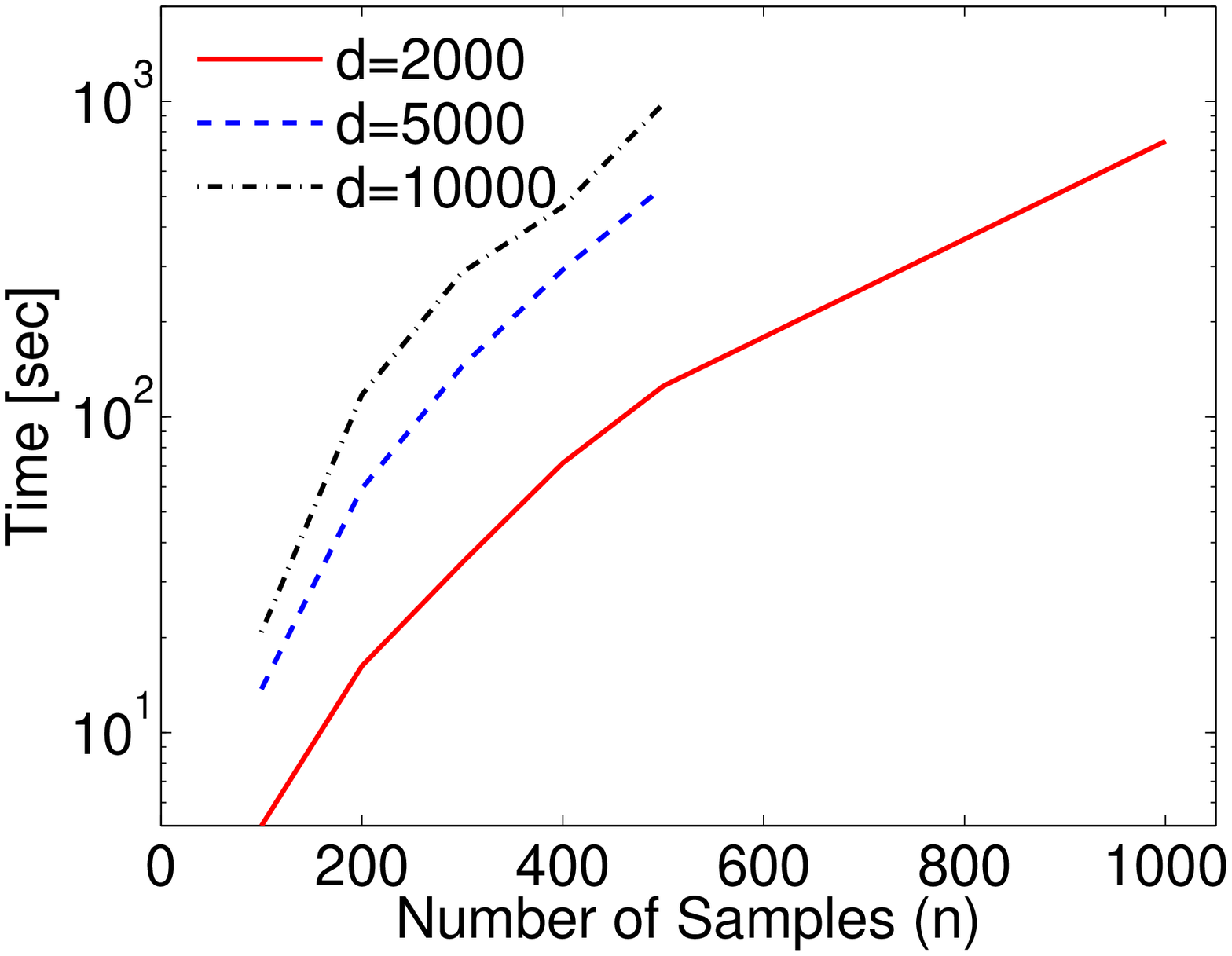}} 
  (b) Computational time. \\ \vspace{-0.10cm}
\end{minipage}
\caption{Result of synthetic data. (a): The regularization path for the synthetic data. (b): Computational time of $\text{N}^3$LARS (without using distributed computing and Nyst\"{o}m approximation) with respect to the dimensionality and the number of samples.} 
    \label{fig:illustrative_example}
\end{center}
\end{figure*}

\section{Experiments}
\label{sec:experiments}
In this section, we evaluate $\text{N}^3$LARS. First we present proof-of-concept experiments on synthetic data.  Next, we perform two-fold evaluations on real-world dataset: (a) accuracy results on small scale datasets, and (b) scalability results on a large-scale biology dataset.

\subsection{Setup}
We compare $\text{N}^3$LARS with the following baselines: (a) mRMR \citep{PAMI:Peng+etal:2005}, (b) QPFS (MI)\citep{JMLR:Rodriguez+etal:2010}, (c) QPFS (HSIC) \citep{DBLP:journals/jmlr/CortesMR12}, (d) forward selection with HSIC (FOHSIC), (e) HSIC Lasso \citep{yamada2014high}, and HITON-PC \citep{aliferis2010local}. Since FOHSIC compares favorably with HSFS and is more computationally efficient than HSFS, we only use FOHSIC in this paper for comparison. For QPFS and mRMR, we use a C++ implementation of the mutual information estimator\footnote{\url{http://penglab.janelia.org/proj/mRMR/}}. Then, the QP solver SeDuMi\footnote{\url{http://sedumi.ie.lehigh.edu/}} is used to solve QPFS. For HSIC Lasso, we used the publicly available Matlab code\footnote{\url{http://www.makotoyamada-ml.com/hsiclasso.html}}. For all experiments involving QPFS (MI) and QPFS (HSIC), we set $\gamma = 0.5$. For large sample sized high-dimensional biology data, we only compare with NHSIC based maximum relevance feature selection (MR-NHSIC) and mRMR as other methods are extremely slow on large datasets. 

\subsection{Synthetic Dataset}
We consider a regression problem from an 2000 dimensional input, where output data is generated according to the following expression:
\[
Y = X_{1}*\exp(X_{2}) + X_{3} + 0.1*E,
\] 
where $(X_1, \ldots, X_{1000})^\top \sim N(\boldzero_{1000},\boldI_{1000})$, $X_{1001} = X_1 + 0.01*E, \ldots, X_{2000} = X_{1000} + 0.01*E$, and $E \sim N(0,1)$. Note, $X_4, \ldots, X_{1000}$ and $X_{1004}, \ldots, X_{2000}$ are irrelevant features to the output, and $X_{1001}$, $X_{1002}$, and $X_{1003}$ are redundant features of $X_1$, $X_2,$, and $X_3$. Here, $N(\boldmu,\boldSigma)$ denotes the multi-variate Gaussian distribution with mean $\boldmu$ and covariance matrix $\boldSigma$. 

Figure~\ref{fig:illustrative_example}(a) shows the regularization path for 10 features, and this illustrates that the proposed method can select non-redundant features. Figure~\ref{fig:illustrative_example}(b) shows the computational time for $\text{N}^3$LARS (without using distributed computing and Nyst\"{o}m approximation) on a Xeon 2.4GHz (16 cores) with 24GB memory. As can be seen, the computational cost of $\text{N}^3$LARS dramatically increases when the number of samples grows. Moreover, since $\text{N}^3$LARS needs $O(dn^2)$ memory space, it is not possible to solve $\text{N}^3$LARS even if the number of samples is 1000. Thus, the Nyst\"{o}m approximation and distributed computation are necessary for the proposed method to solve high-dimensional and large sample cases. In Section \ref{sec:p53}, we show that $\text{N}^3$LARS with distributed computation can extract 100 features from $d = 5000$ and $n = 26120$ data in a few hours.

\subsection{Accuracy results (for large $d$, small $n$)}
First we qualitatively compare our proposed method on high-dimensional dataset (e.g., $d > 10^4$ and $n < 500$). As most baselines are computationally slow, we restrict this comparative study to small scale datasets. In next section, we show scalability comparison for a subset of the baselines on a larger dataset.

\vspace{.1in}
\noindent {\bf Classification:}
We use the ASU feature selection datasets\footnote{\url{http://featureselection.asu.edu/datasets.php}} for evaluating the performance of the proposed method in a high-dimensional feature selection setting (See Table~\ref{tab:feat_data} for details). 

\begin{table}[t]
\centering
\caption{Summary of Real-world Datasets.}
\label{tab:feat_data}
\small
\begin{tabular}{|c|l|c|c|}
\hline
Type & Dataset & Features ($d$) & Samples ($n$)  \\ \hline
 & AR10P    & 2400    & 130  \\
& PIE10P   & 2400    & 210  \\
Small \& High& PIX10P   & 10000   & 100   \\
(Classification)& ORL10P   & 10000   & 100   \\ 
 & TOX      & 5748    & 171  \\
& CLL-SUB  & 11340   & 111   \\  \hline
Small \& High & \multirow{2}{*}{TRIM32} & \multirow{2}{*}{31098} & \multirow{2}{*}{120} \\ 
(Regression)& & & \\\hline
Large \& High & \multirow{2}{*}{p53} & \multirow{2}{*}{5408} & \multirow{2}{*}{31420} \\ 
(Classification)& & & \\\hline
\end{tabular}
\end{table}

For this classification experiment, we use 80$\%$ of samples for training and the rest for testing. We run the classification experiments 100 times by randomly selecting training and test samples and report the average classification accuracy.  Since all datasets are multi-class, we use multi-class kernel logistic regression (KLR) \citep{book:Hastie+Tibshirani+Friedman:2001,yamada2010semi}. For KLR we use Gaussian kernel where the kernel width and the regularization parameter are chosen based on a 3-fold cross-validation. For all experiments, we first select 50 features by feature selection methods on training data and then use top $m = 10,20,\ldots,50$ features having the largest absolute regression coefficients. 

To check whether $\text{N}^3$LARS can successfully select non-redundant features, we use  \emph{redundancy rate} (RED) \citep{AAAI:Zheng+etal:2010}, $\textnormal{RED} = \frac{1}{m(m-1)} \sum_{\boldu_k, \boldu_j, k > l} |\rho_{k,l}|,$ 
 where $\rho_{k,l}$ is the correlation coefficient between the $k$-th and $l$-th features. A large RED score means that selected features are more correlated to each other, that is, many redundant features are selected. 

\begin{figure*}[t!]
\begin{center}
\begin{minipage}[t]{0.325\linewidth}
\centering
  {\includegraphics[width=0.99\textwidth]{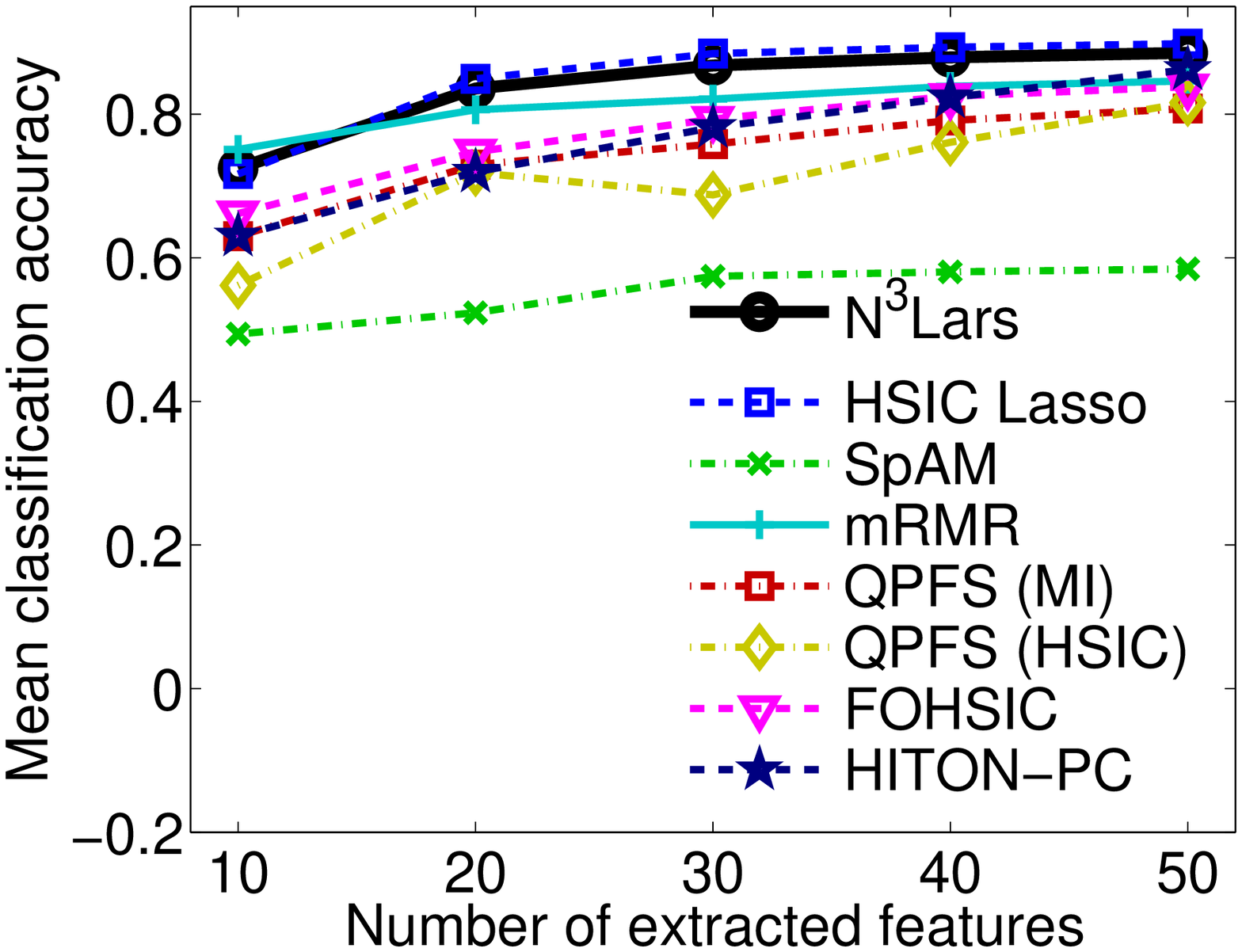}} \\ \vspace{-0.10cm}
(a) AR10P
\end{minipage}
\begin{minipage}[t]{0.325\linewidth}
\centering
  {\includegraphics[width=0.99\textwidth]{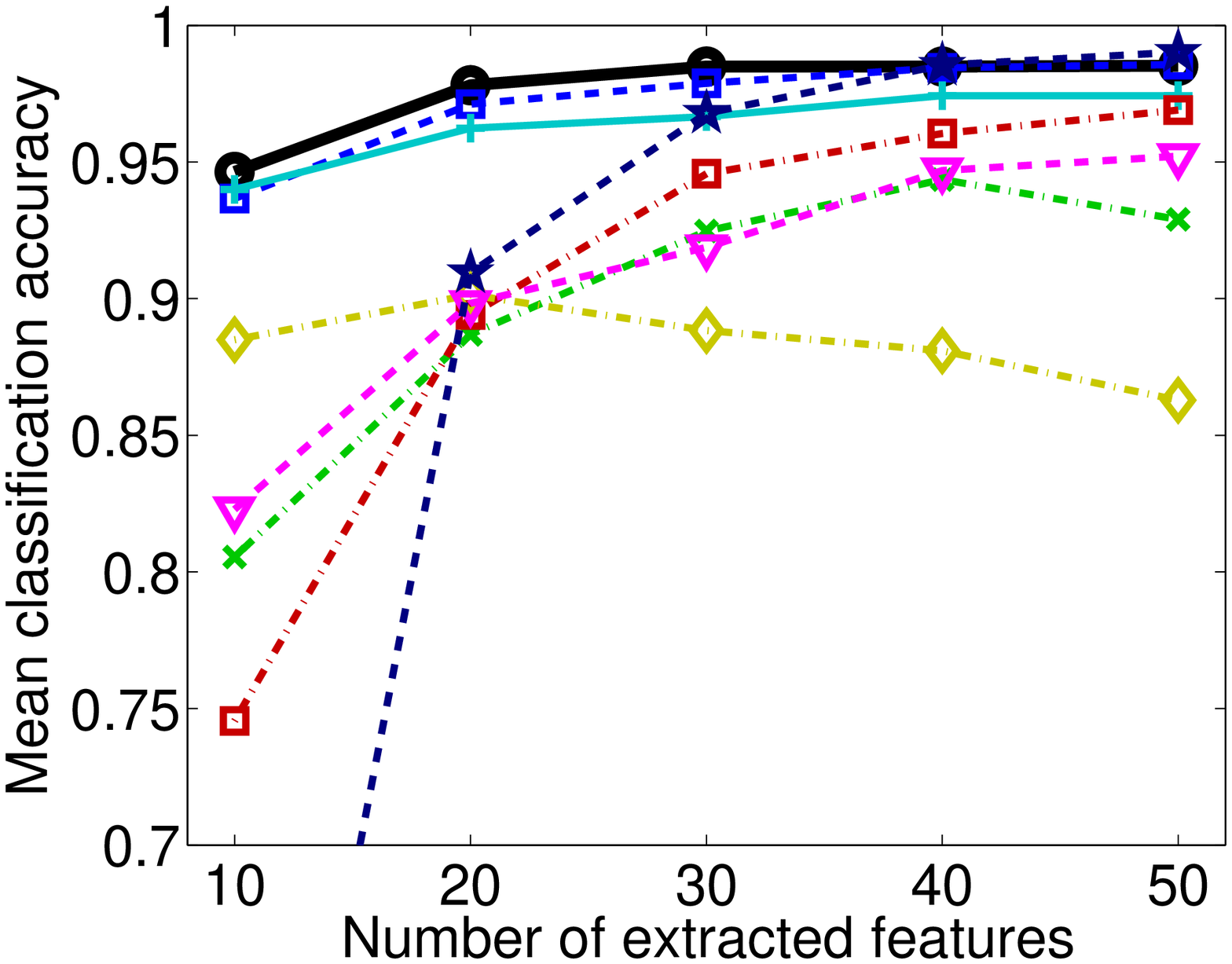}} \\ \vspace{-0.10cm}
(b) PIE10P
\end{minipage}
  \begin{minipage}[t]{0.325\linewidth}
\centering
{\includegraphics[width=0.99\textwidth]{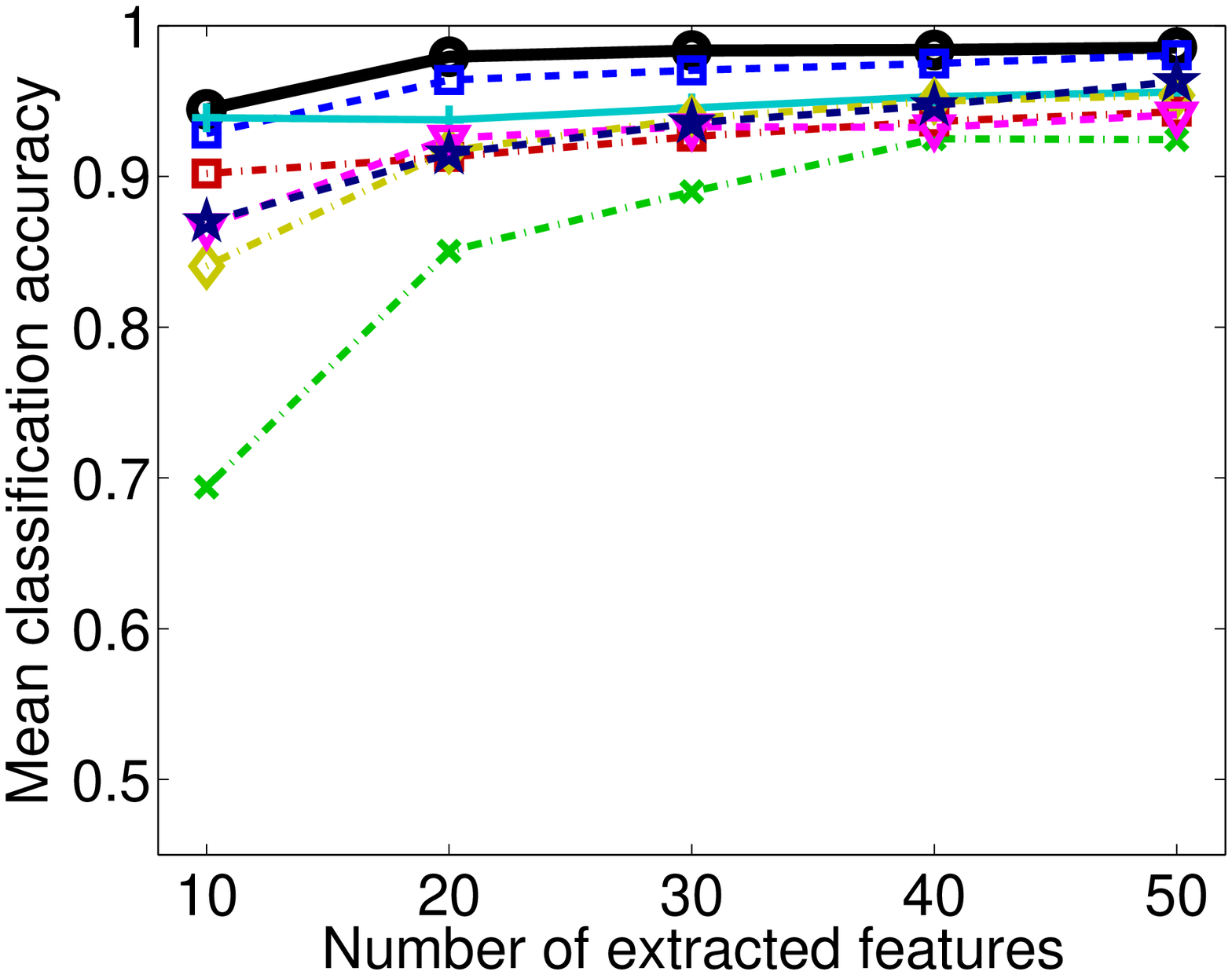}} \\  \vspace{-0.10cm}
(c) PIX10P
  \end{minipage} \\
\begin{minipage}[t]{0.325\linewidth}
\centering
  {\includegraphics[width=0.99\textwidth]{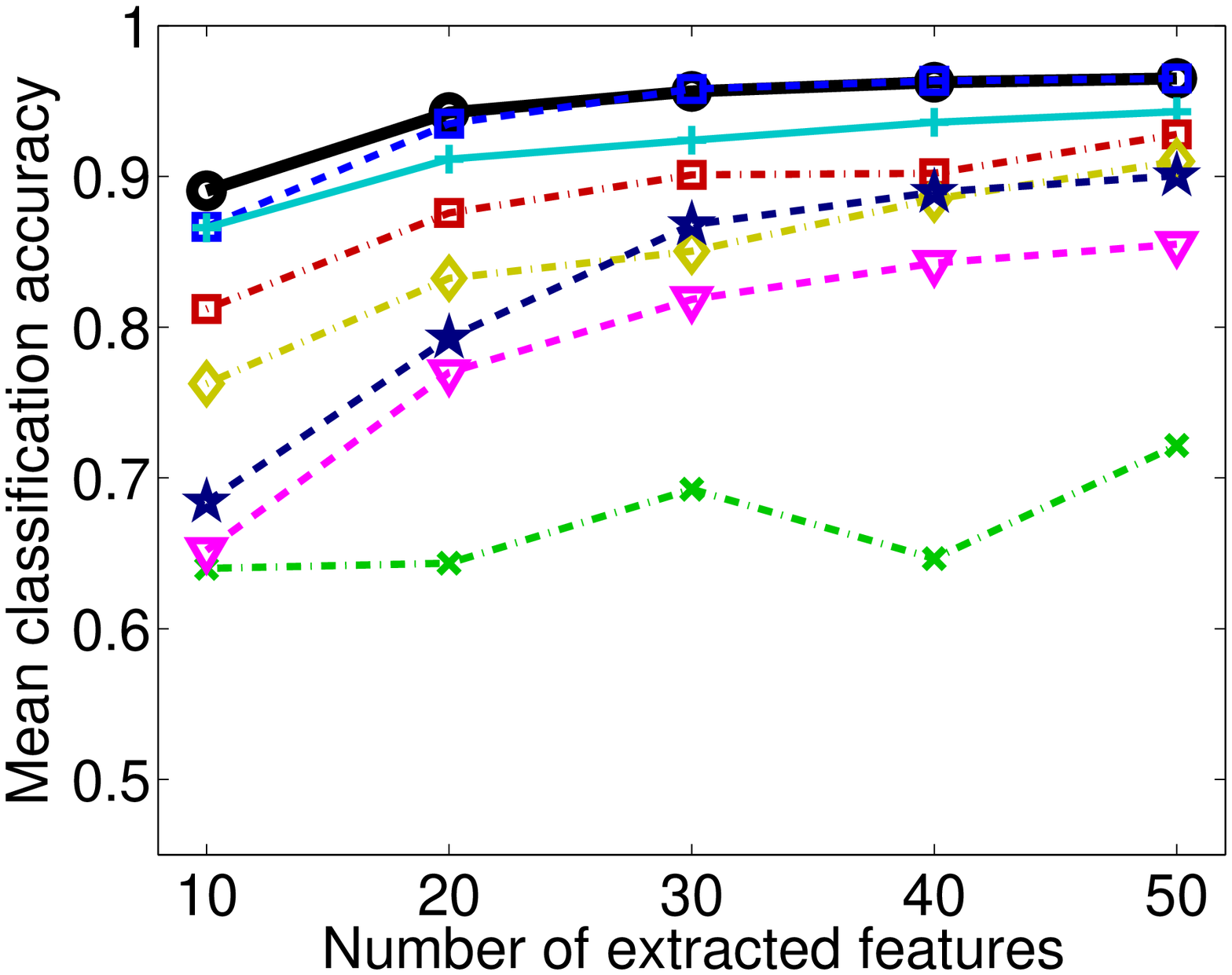}} \\ \vspace{-0.10cm}
(d) ORL10P
\end{minipage}
\begin{minipage}[t]{0.325\linewidth}
\centering
  {\includegraphics[width=0.99\textwidth]{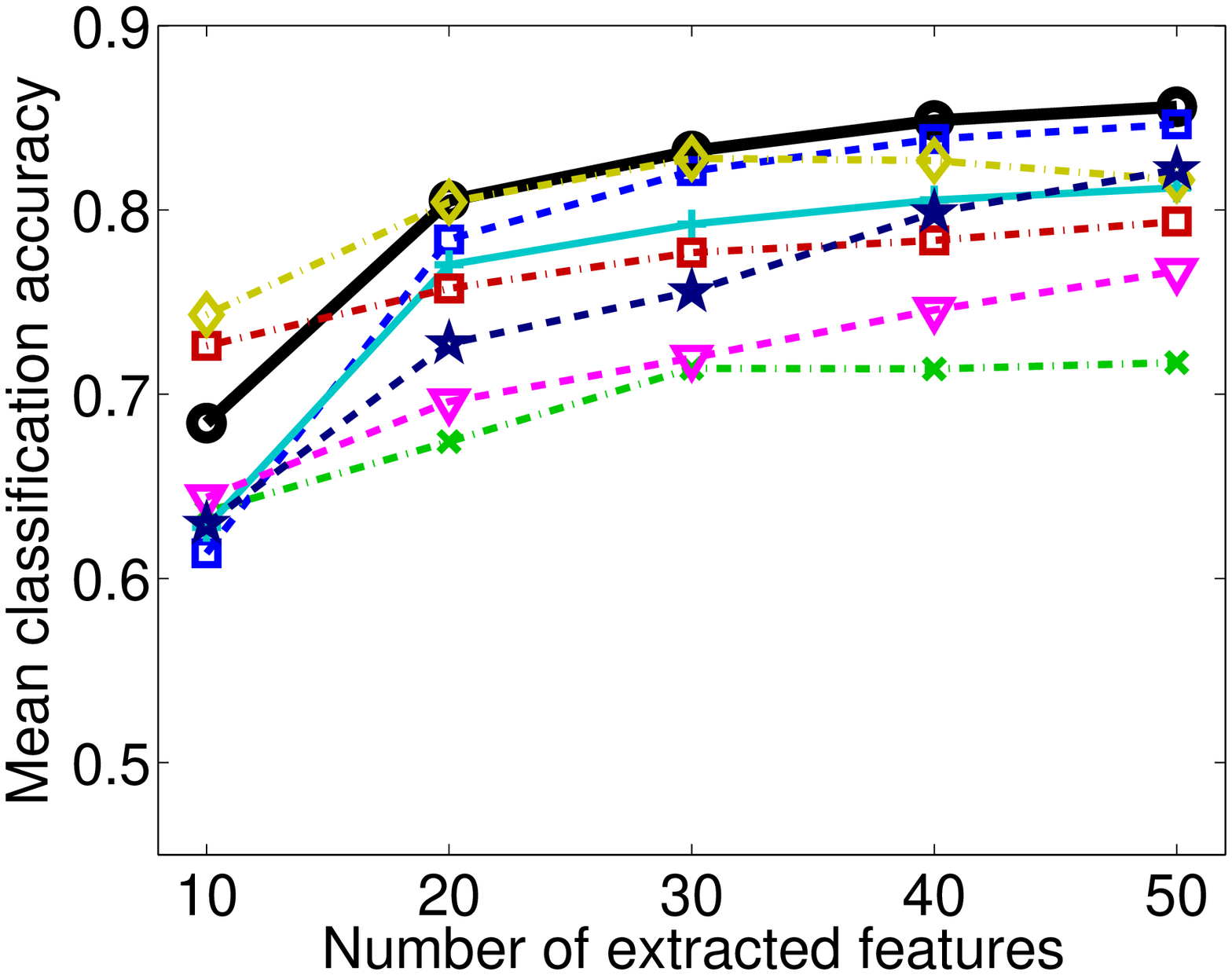}} \\ \vspace{-0.10cm}
(e) TOX
\end{minipage}
  \begin{minipage}[t]{0.325\linewidth}
\centering
{\includegraphics[width=0.99\textwidth]{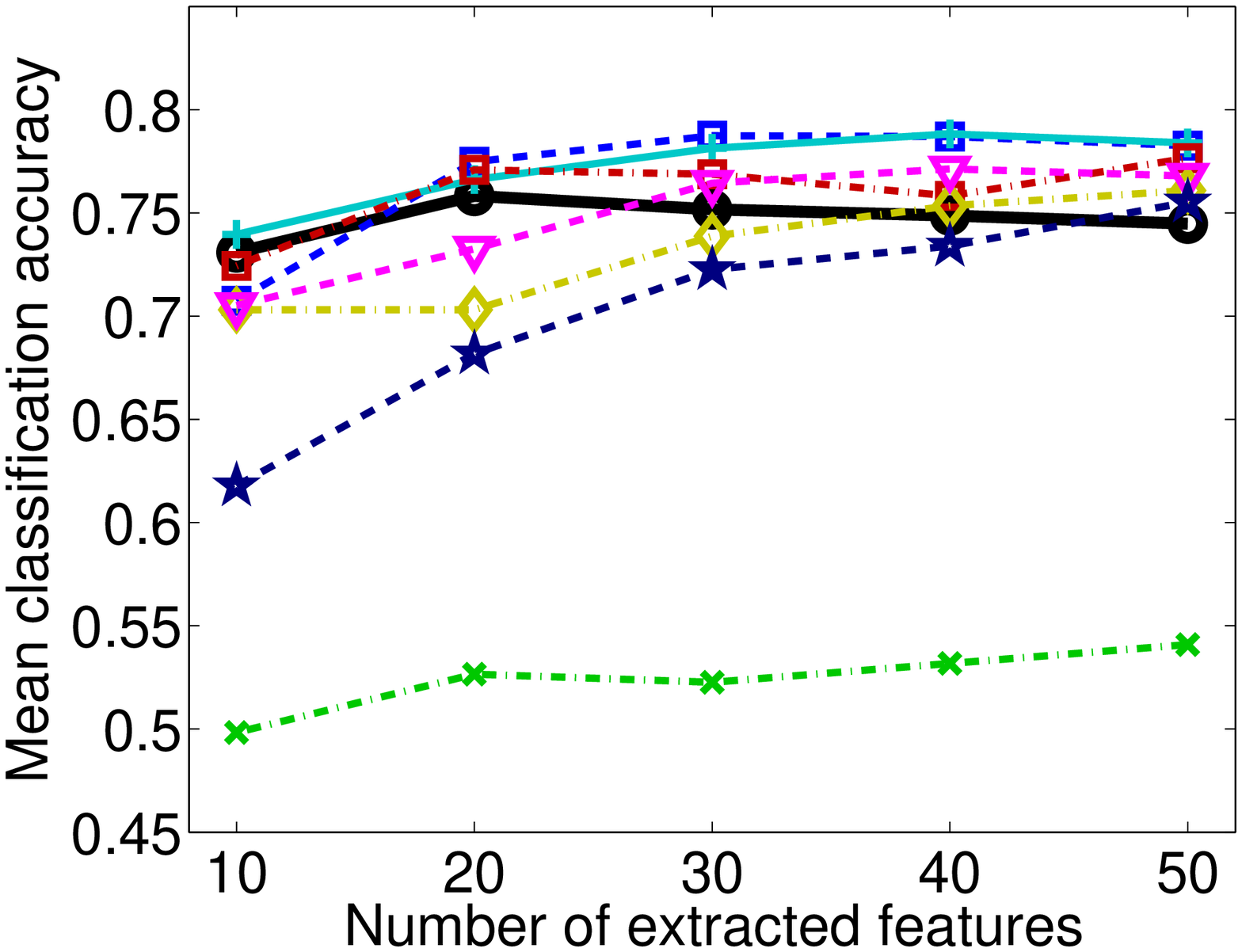}} \\  \vspace{-0.10cm}
(f) CLL-SUB
  \end{minipage}
 \caption{Mean classification accuracy for real-world data. The horizontal axis denotes the number of selected features, and the vertical axis denotes the mean classification accuracy.} 
    \label{fig:result_ASU}
\end{center}
\end{figure*}

\begin{table*}[t]
\centering
\caption{Mean Redundancy Rate (with Standard Deviations in
Brackets) for Real-world Data. }
\label{table:redundancy_rate}
\begin{tabular}{c l@{\ }r@{\ }|l@{\ }r@{\ }|l@{\ }r@{\ }|l@{\ }r@{\ }|l@{\ }r@{\ }|l@{\ }r@{\ }|l@{\ }r@{\ }}
\hline
\multicolumn{1}{c|}{\multirow{2}{*}{Dataset}}  & \multicolumn{2}{c|}{\multirow{2}{*}{$\text{N}^3$LARS}}  & \multicolumn{2}{c|}{{HSIC}} &  \multicolumn{2}{c|}{\multirow{2}{*}{SpAM}}  &   \multicolumn{2}{c|}{\multirow{2}{*}{mRMR}}   & \multicolumn{2}{c|}{{QPFS}} & \multicolumn{2}{c|}{\multirow{2}{*}{FOHSIC}} & \multicolumn{2}{c}{\multirow{2}{*}{HITON-PC}}\\
\multicolumn{1}{c|}{}  & \multicolumn{2}{c|}{} & \multicolumn{2}{c|}{{Lasso}}  &  \multicolumn{2}{c|}{} & \multicolumn{2}{c|}{}  & \multicolumn{2}{c|}{(HSIC)} & \multicolumn{2}{c|}{} & \multicolumn{2}{c}{}\\
\hline  
\multicolumn{1}{l|}{AR10P}     & {\bf .154} & (.016) &  .196  & (.028)    & .255 & (.036) & .268 & (.038) &  .235 & (.034) & .350 & (.091) & .201 & (.035)\\
\multicolumn{1}{l|}{PIE10P}     & {\bf .124} & (.010) &  .135  & (.014)  & .250 & (.042)  & .225 & (.036)  & .155 & (.021) & .285 & (.059)  & .174 & (.029)\\
\multicolumn{1}{l|}{PIX10P}    & {\bf .171} & (.020) &  .177 & (.023)   &  .388 & (.105)  & .200 & (.066) & .198 & (.036) & .348 & (.064) & .210 & (.037)\\
\multicolumn{1}{l|}{ORL10P}   & {\bf .182} & (.022) & .192  & (.026)   & .300 & (.047)  & .294 & (.095)  & .191 & (.034) & .225 & (.045)  & .173 & (.021)\\
\multicolumn{1}{l|}{TOX}        &  .397 & (.029) & .382        & (.027)    & .391 & (.028)  & .386 & (.032) & {\bf .371} & (.040) & .396 & (.036) & .419 & (.032) \\
\multicolumn{1}{l|}{CLL-SUB}  & .349 & (.039)& .344        & (.034)    & .403 & (.058)  & .328 & (.039) & {\bf .281} & (.050) & .352 & (.061) & .364 & (.036)\\
 \hline
\end{tabular}
\end{table*}

Figure~\ref{fig:result_ASU} shows the average classification accuracy over 100 runs, where $x$-axis is the number of selected features.  As can be observed, the proposed method compares favorably with the HSIC Lasso, a state-of-the-art high-dimensional feature selection method. Table~\ref{table:redundancy_rate} shows the averaged RED values over top $m=10, 20, \ldots, 50$ features selected by each feature selection method.  The RED score of $\text{N}^3$LARS tends to be smaller than those of existing methods. Overall, the proposed method can select non-redundant features. 

\vspace{.1in}
\noindent {\bf Regression:}
Next, we evaluate our proposed method using the Affymetric GeneChip Rat Genome 230 2.0 Array data set \citep{scheetz2006regulation}. In this dataset, there are 120 rat subjects with 31098 genes (i.e, $n=120,~d=31098$),  which were measured from eye tissue. In this paper, we focus on finding genes that are related to the TRIM32 gene \citep{scheetz2006regulation,huang2010variable}, which was recently found to cause the Bardet-Biedl syndrome and takes real values. 

For this regression experiment, we use 80$\%$ of samples for training and the rest for testing. As earlier, we run the regression experiments 100 times by randomly selecting training and test samples, and compute the average mean squared error (MSE). We employ  kernel regression (KR) \citep{book:Schoelkopf+Smola:2002} with the Gaussian kernel for evaluating the mean squared error. Similar to the multi-class classification problem, we first choose 50 features using feature selection methods on training data and then use top $m = 10,20,\ldots,50$ features having the largest absolute regression coefficients. In KR, the Gaussian width and the regularization parameter are chosen based on 3-fold cross-validation.  Note, in this experiments, most existing methods are too slow to finish. Thus, we only included the $\text{N}^3$LARS, HSIC Lasso, linear Lasso, and  mRMR results.

Figure~\ref{fig:result_TRIM32} shows the mean squared error over 100 runs as a function of the number of selected features. As can be observed, the proposed method performs as good as HSIC Lasso and mRMR. 
  we use HSIC as a dependency measure instead of NHSIC, we can obtain the HSIC Lasso solution. In this regression experiment, HSIC measure performs better than NHSIC. 

\begin{figure}[t!]
\begin{center}
\begin{minipage}[t]{0.5\linewidth}
\centering
  {\includegraphics[width=0.99\textwidth]{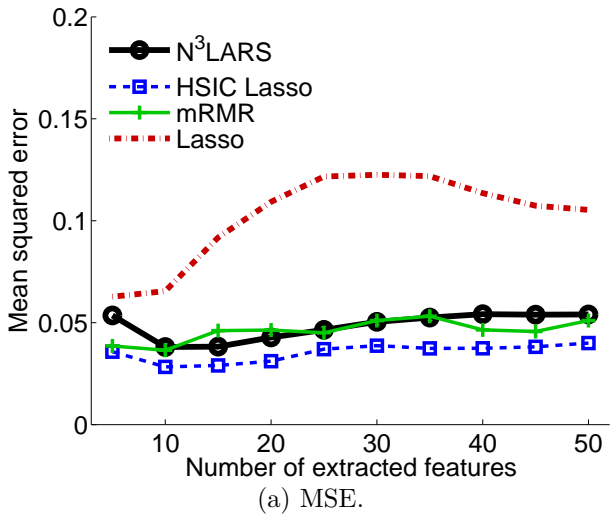}} \\ \vspace{-0.10cm}
(a) MSE.
\end{minipage}
 \caption{(a): Mean squared error for the TRIM32 data. The horizontal axis denotes the number of selected features, and the vertical axis denotes the mean squared error.  The average redundancy rate of the proposed method, HSIC Lasso, mRMR, and Lasso are 0.42, 0.45, 0.39, and 0.43, respectively. }
    \label{fig:result_TRIM32}
\end{center}
\vspace{-0.1in}
\end{figure}

\begin{figure}[t!]
\begin{center}
\begin{minipage}[t]{0.48\linewidth}
\centering
  {\includegraphics[width=0.99\textwidth]{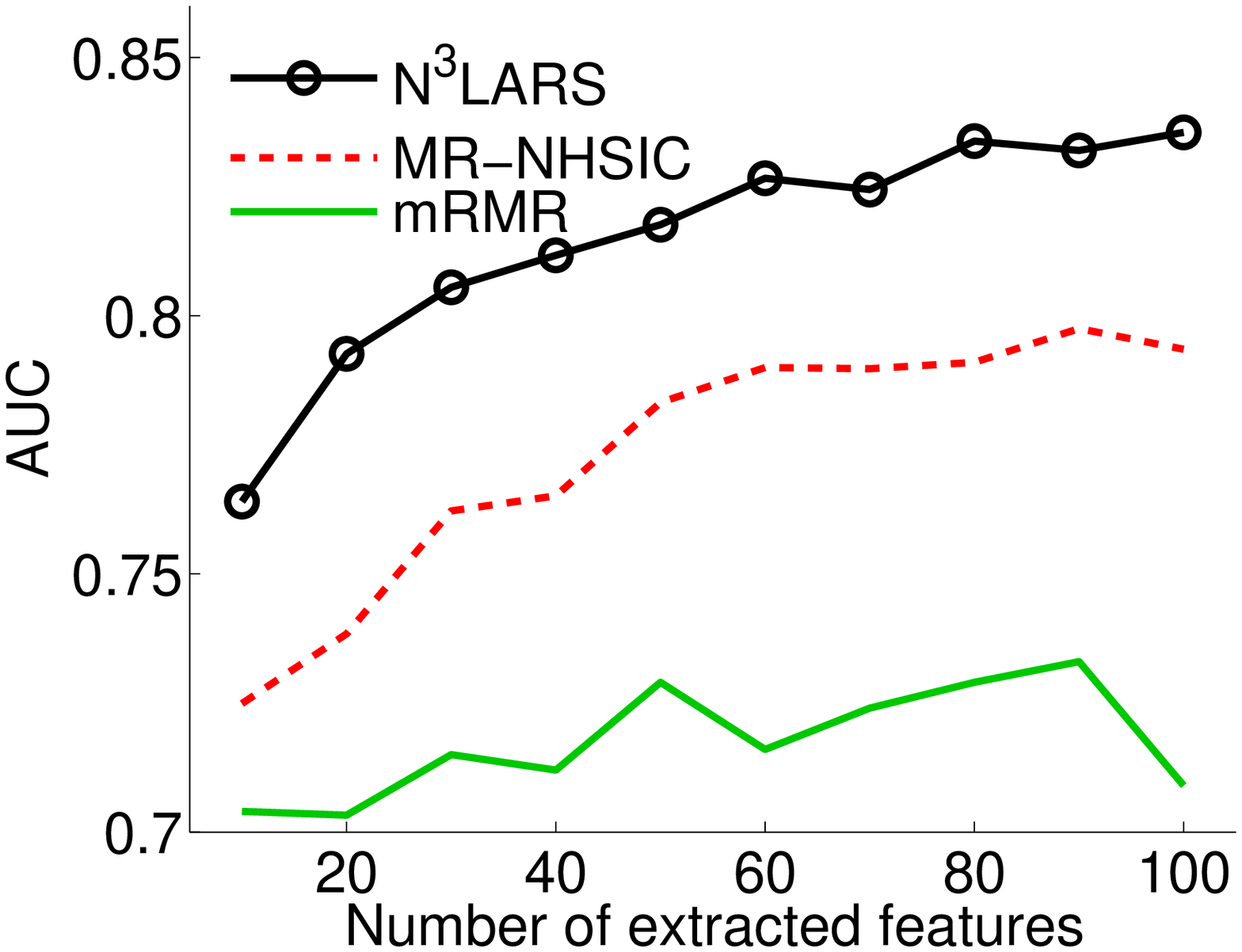}} (a) 
\end{minipage}
\begin{minipage}[t]{0.48\linewidth}
\centering
  {\includegraphics[width=0.99\textwidth]{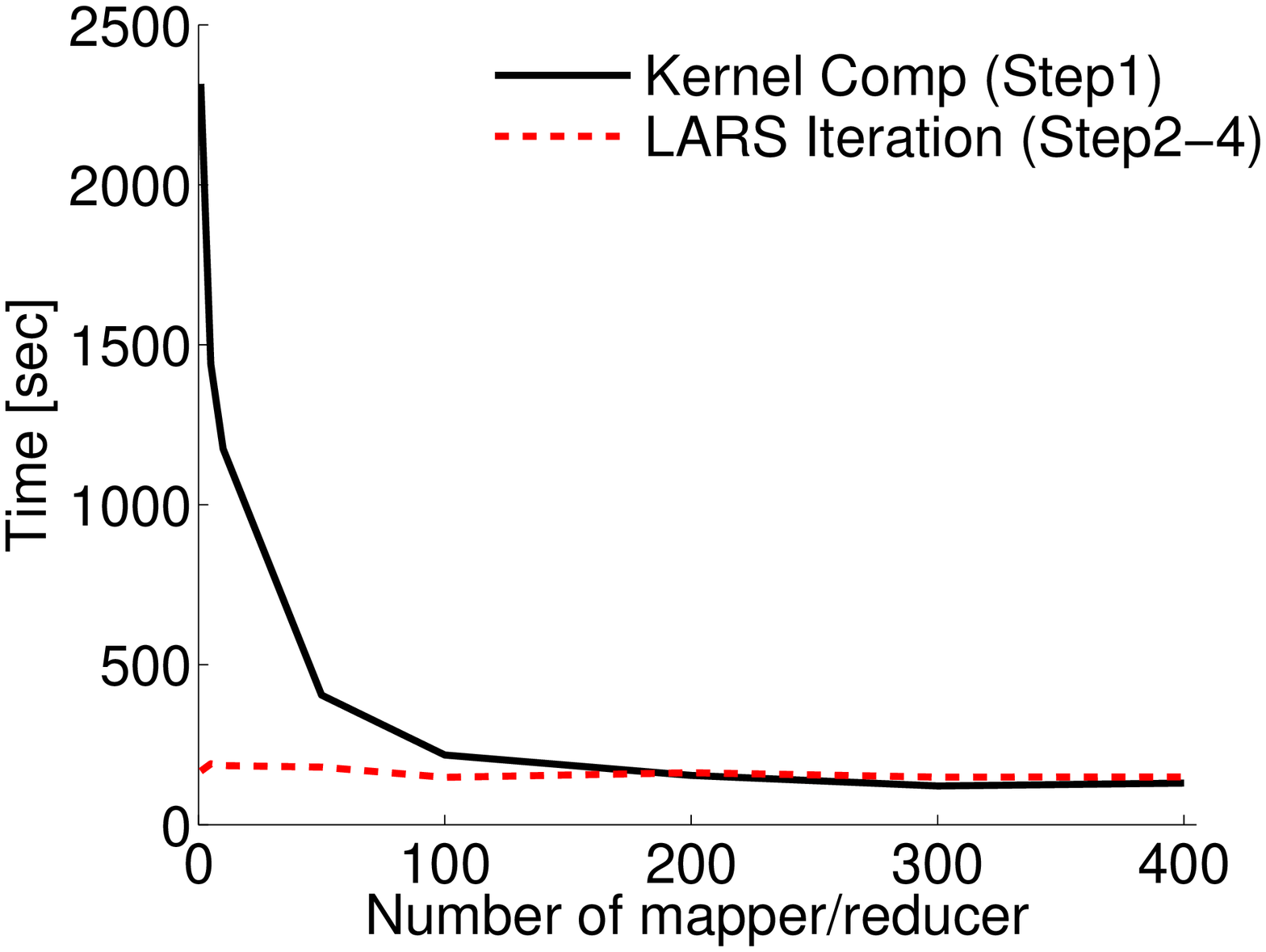}} (b) 
\end{minipage}
\caption{Results for a large and high-dimensional biology data. Comparable methods according to the paired \emph{t-test} at the significance level 5\% are specified by '$\circ$'. (a): The AUC score. (b): Computation time.} 
    \label{fig:lars_p53}
\end{center}
\vspace{-0.2in}
\end{figure}

\subsection{Scalability results (for large $d$, large $n$)}
\label{sec:p53}
In this section, we evaluate the proposed method for a large and high-dimensional setting. We use the p53 mutant data set \citep{danziger2009predicting}, which consists of 5408 features and 31420 samples. In the context of existing non-linear feature selection methods, p53 dataset is quite big and hence has been used to demonstrate scalability results in our paper. The goal here is to predict p53 transcriptional activities (i.e., active or inactive) from the data, where all class labels are determined via in vivo assays \citep{danziger2009predicting}. Note that, the dataset is dense, i.e., most of features take non-zero values.  
For this experiment, we use 26120 samples for training and 5000 samples for test.  We run the classification experiments 5 times by randomly selecting training and test samples, and report the average area under the ROC curve (AUC) score. In this experiment, we use the gradient boosting decision tree (GBDT)~\citep{friedman2001greedy} as the classifier, and use 100 trees with 20 nodes. We first selected 100 features by feature selection methods and then use top $m = 10,20,\ldots,100$ features having the largest absolute regression coefficients. Since the data is large, we compare $\text{N}^3$LARS with NHSIC based maximum relevance (MR-NHSIC) and mRMR\footnote{The mRMR package uses sub-sampling technique to handle large-scale data.}. 

Figure \ref{fig:lars_p53}(a) shows that the AUC scores for $\text{N}^3$LARS, MR-NHSIC, and mRMR, respectively. It is interesting that the simple  MR-NHSIC outperforms mRMR in this experiment, and this indicates that the mutual information in mRMR is not accurately estimated. Overall, the proposed method outperforms the existing methods. Figure \ref{fig:lars_p53}(b)  shows the computational time for $\text{N}^3$LARS with respect to the number of mappers/reducers, where we changed the number of mappers/reducers as 1, 5, 10, 50, 100, 200, 300, and 400. The solid black line indicates the computation time for $\boldF_k$ and $\boldG$ (Step1). The dotted line is the computation time for each LARS iteration (Step2 - 4). As can be seen, the computational time for the Step1  dramatically decreases as the number of mappers/reducers increases. 

\section{Conclusion}
\label{sec:conclusion}
In this paper, we proposed a novel non-linear feature selection method called the \emph{Nonlinear Non-Negative least angle regressions} ($\text{N}^3$LARS). 
Our proposed method can efficiently obtain a globally  optimal solution by exploiting a non-negative variant of the LARS algorithm where the similarity between input and output is measured through normalized HSIC (NHSIC). Furthermore, we proposed a distributed computation framework for $\text{N}^3$LARS that solves a large and high-dimensional feature selection problem in reasonable time, and experimental results demonstrated that $\text{N}^3$LARS is promising.



\bibliography{main}
\bibliographystyle{natbib}
\end{document}